\documentclass{article} 
\usepackage{nips15submit_e,times}
\usepackage{hyperref}
\usepackage{url}

\usepackage{times}
\usepackage{helvet}
\usepackage{courier}
\usepackage{amssymb}
\usepackage{amsmath, amsthm}
\usepackage{color}
\usepackage{graphicx}
\usepackage{verbatim}
\usepackage{tabularx}
\usepackage{amsfonts,amsmath,amsthm,array,amssymb}
\usepackage{graphicx}      
\usepackage{times}
\usepackage{epstopdf}
\usepackage{url}
\usepackage{algorithm, algorithmicx, algpseudocode}
\usepackage{subfig}

\usepackage{multirow}

\def\E{\mathbb{E}}
\def\P{\mathbb{P}}

\newtheorem{theorem}{Theorem}

\newtheorem{proposition}{Proposition}

\title{Tight Variational Bounds via Random Projections and I-Projections}

\author{
Lun-Kai Hsu, Tudor Achim, Stefano Ermon \\
Department of Computer Science\\
Stanford University\\
Stanford, CA 94305 \\
\texttt{\{luffykai,tachim,ermon\}@cs.stanford.edu} \\
}

\nipsfinalcopy 

\newcommand\blfootnote[1]{%
  \begingroup
  \renewcommand\thefootnote{}\footnote{#1}%
  \addtocounter{footnote}{-1}%
  \endgroup
}

\begin{document}
\blfootnote{LKH and TA contributed equally to this paper. }

\maketitle

\begin{abstract}
Information projections are the key building block of variational inference algorithms and are used to approximate a target probabilistic model by projecting it onto a family of tractable distributions. In general, there is no guarantee on the quality of the approximation obtained. To overcome this issue, we introduce a new class of random projections to reduce the dimensionality and hence the complexity of the original model. In the spirit of random projections, the projection preserves (with high probability) key properties of the target distribution.
We show that information projections can be combined with random projections to obtain provable
guarantees on the quality of the approximation obtained, regardless of the complexity of the
original model. We demonstrate empirically that augmenting mean field with a random projection step
dramatically improves partition function and marginal probability estimates, both on synthetic and real world data.
\end{abstract}

\section{Introduction}
\label{introduction}
Probabilistic inference is a core problem in machine learning, physics, and statistics~\cite{koller2009probabilistic}. Probabilistic inference methods are needed for training, evaluating, and making predictions with probabilistic models~\cite{murphy2012machine}. 
Developing scalable and accurate inference techniques is the key computational bottleneck towards
deploying large-scale statistical models, but exact inference is known to be computationally intractable. The root cause is the curse of dimensionality -- the number of possible scenarios to consider grows exponentially in the number of variables, and in continuous domains, the volume grows exponentially in the number of dimensions~\cite{bellman1961adaptive}. Approximate techniques are therefore almost always used in practice~\cite{murphy2012machine}. 

Sampling-based techniques and variational approaches are the two main paradigms for approximate inference~\cite{andrieu2003introduction,wainwright2008graphical}. 
Sampling-based approaches attempt to approximate intractable, high-dimensional distributions using a (small) collection of representative samples~\cite{gogate2011samplesearch,ermon2013embed,maddison2014sampling}. 
Unfortunately, it is usually no easier to obtain such samples than it is to solve the original probabilistic inference problem~\cite{jerrum1997markov}. 
Variational approaches, on the other hand, approximate an intractable distribution using a family of tractable ones. 
Finding the best approximation, also known as computing an \emph{I-Projection} onto the family, is the key ingredient in all variational inference algorithms. 
In general, there is no guarantee on the quality of the approximation obtained. Intuitively, if the target model is too complex with respect to the family used, then the approximation will be poor. 

To overcome this issue, we introduce a new class of \emph{random} projections~\cite{vadhan2011pseudorandomness,goldreich2001randomized,uai13LPCount}. 
These projections take as input a probabilistic model and randomly perturb it, reducing its degrees of freedom. 
The projections can be computed efficiently and they reduce the effective dimensionality and complexity of the target model. 
Our key result is that the randomly projected model can then be approximated with I-projections onto simple families 
of distributions such as mean field with provable guarantees on the accuracy, regardless of the complexity of the original model.
Crucially, in the spirit of random projections for dimensionality reduction, the random projections affect key properties of the target distribution (such as the partition function) in a highly predictable way. 
The I-projection of the projected model can therefore be used to accurately recover properties of the original model with high probability.

We demonstrate the effectiveness of our approach by using mean field augmented
with random projections to estimate marginals and the log partition function 
on models of synthetic and real-world data, empirically showing large 
improvements on both tasks.

\section{Preliminaries}
\label{preliminaries}
Let $p(x) = \frac{1}{Z} \prod_\alpha \psi_\alpha(\{x\}_\alpha)$ be a probability distribution over $n$ binary variables $x \in \{0,1\}^n$ specified by an undirected graphical model
\footnote{ We restrict ourselves to binary variables for the ease of exposition. 
Our approach applies more generally to discrete graphical models.}
~\cite{koller2009probabilistic}. We further assume that $p(x)$ is a member of an exponential family of distributions parameterized by
$\theta \in \mathbb{R}^d$ and with sufficient statistics 
$\phi(x)$~\cite{wainwright2008graphical}, i.e., 
$p(x) = \frac{\exp(\theta\cdot \phi(x))}{Z}$.
The constant 
$Z = \sum_{x \in \{0,1\}^n} \prod_{\alpha \in \mathcal{I}} 
    \psi_\alpha(\{x\}_\alpha)$
is known as the partition function and ensures that the 
probability distribution is properly normalized. 
Computing the partition function is one of the key computational challenges 
in probabilistic reasoning and statistical machine learning, as it is needed to 
evaluate likelihoods and compare competing models of data. 
This computation is known to be intractable (\#-P hard) in the worst-case, 
as the sum is defined over an exponentially large number of 
terms~\cite{valiant1979complexity,koller2009probabilistic}. 
We focus on variational approaches for approximating the partition function.

\subsection{Variational Inference and I-projections}
The key idea of variational inference approaches is to approximate the intractable 
probability distribution $p(x)$ with one that is more tractable. The approach is 
to define a family $\mathcal{Q}$ of tractable distributions and then look for a 
distribution in this family that minimizes a notion of divergence from $p$. 
Typically, the Kullback-Leibler divergence $D_{KL} (q||p)$ is used, which is 
defined as follows
\begin{equation}
\label{kldiv}
D_{KL} (q||p)  = \sum_x q(x) \log \frac{q(x)}{p(x)} = \sum_x q(x) \log q(x) - \theta \cdot \sum_x q(x) \phi(x) + \log Z 
\end{equation}
A distribution $q^\ast \in \mathcal{Q}$ that minimizes this divergence, $q^\ast = \arg \min_{q \in Q} D_{KL} (q||p)$, 
is called an information projection (\emph{I-projection}) onto $\mathcal{Q}$. Intuitively, $q^\ast$ is the ``closest'' 
distribution to $p$ among all the distributions in $\mathcal{Q}$. Typically, one chooses $\mathcal{Q}$ to be a 
family of tractable distributions for which inference is tractable, i.e., such that (\ref{kldiv}) can be evaluated efficiently. 
The simplest choice, which removes all conditional dependencies, is to let $\mathcal{Q}$ be the set of fully 
factored probability distributions over $\mathcal{X}$, namely $\mathcal{Q}_{MF} = \{q(x)| q(x) =
\prod_i q_i(x_i)\}$. This is known as the mean field approximation.
Even when $\mathcal{Q}$ is tractable, computing an I-projection, i.e., \emph{minimizing} the KL divergence, is a 
non-convex optimization problem which can be difficult to solve.

Since the KL-divergence is non-negative, equation (\ref{kldiv}) shows that any distribution $q \in \mathcal{Q}$ 
provides a lower bound on the value of the partition function
\begin{equation}
\label{varLB}
\log Z \geq \max_{q \in \mathcal{Q}} - \sum_x q(x) \log q(x) + \theta \cdot \sum_x q(x) \phi(x)
\end{equation}

The distribution $q^\ast$ that minimizes $D_{KL} (q||p)$ is also the distribution that provides the tightest lower 
bound on the partition function by maximizing the RHS of equation (\ref{varLB}). The larger the set 
$\mathcal{Q}$ is, the better $q^\ast$ can approximate $p$ and the tighter the bound becomes. 
If $\mathcal{Q}$ is so large that $p \in \mathcal{Q}$, then $\min_{q \in Q} D_{KL} (q||p) =0 $, 
because when $q^\ast=p$, $D_{KL} (q^\ast||p)= 0$. In general, however, there is no guarantee on 
the tightness of bound (\ref{varLB}) even if the optimization can be solved exactly.

\subsection{Random Projections}
We introduce a different class of \emph{random} projections that we will use for probabilistic inference. 
Let $\mathcal{P}$ be the set of all probability distributions over $\{0,1\}^n$. We introduce a 
family of operators $\mathcal{R}^m_{A,b} : \mathcal{P} \rightarrow \mathcal{P}$, where 
$m \in [0,n]$, $A \in \{0,1\}^{m \times n}$, and $b\in \{0,1\}^m$. 
$R^m_{A,b} \in \mathcal{R}$ maps 
$p(x) = \frac{1}{Z} \prod_\alpha \psi_\alpha(\{x\}_\alpha)$ to a new 
probability distribution $R^m_{A,b}(p)$ restricted to $\left\{ x : Ax=b\mod 2 \right\}$ 
whose probability mass function is proportional to $p$. 
Formally,
\[
R^m_{A,b}(p) (x) = \frac{1}{Z(A,b)}\prod_\alpha \psi_\alpha(\{x\}_\alpha)
\]
In other words, for all $x \in \{x| A x = b \mod 2 \}$ $R^m_{A,b}(p)$ is proportional to the original $p(x)$, and the new normalization constant is
\[
Z(A,b) = \sum_{x | A x = b \bmod 2} \prod_\alpha \psi_\alpha(\{x\}_\alpha)
\]
These operators are clearly idempotent and can thus be interpreted as projections. 

When the parameters $A,b$ are chosen randomly, the operator $R^m_{A,b}$ can be seen as a random projection. 
We consider random projections obtained by choosing $A \in \{0,1\}^{m \times n}$ and $b\in \{0,1\}^m$ 
independently and uniformly at random, i.e., choosing each entry by sampling an independent unbiased Bernoulli random variable. This can be shown to implement a strongly universal hash 
function~\cite{vadhan2011pseudorandomness,goldreich2001randomized}. Intuitively, the projection 
randomly subsamples the original space, selecting configurations $x\in \{0,1\}^n$ pairwise independently with probability $2^{-m}$. 
It can be shown~\cite{uai13LPCount,ermon2014low} that
\[
\E[Z(A,b)] = 2^{-m} Z
\]
where the expectation is over the random choices of $A,b$, and that
$
    Var\left[Z(A,b)\right] = 
        \frac{1}{2^m}\left( 1-\frac{1}{2^m} \right)
        \sum_x \left( \prod_\alpha \psi_\alpha(\{x\}_\alpha) \right)^2
$. As we will formalize later, this random projection simplifies the 
model without losing too much information because it affects the partition function in a highly predictable way (knowing the expectation and the variance is sufficient to achieve high probability bounds). 

To gain some intuition on the effect of the random projection, we can 
rewrite the linear system $A x = b \bmod 2$ 
in reduced row-echelon form~\cite{uai13LPCount}. Assuming $A$ is full-rank, 
we have that $C = [ I_m | A']$ where $I_m$ is the $m \times m$ identity matrix. The system $A x = b$ is mathematically equivalent to $C x = b'$. For notational simplicity, we continue to use $b$ instead of $b'$. 
We can equivalently rewrite the constraints $A x = b \bmod 2$ as the following set of constraints
\begin{align*}
x_1 = \bigoplus_{i=m+1}^n c_{1i} x_i \oplus b_1,  \cdots, x_m = \bigoplus_{i=m+1}^n c_{mi} x_i \oplus b_m \\
\end{align*}
where $\oplus$ denotes the exclusive-or (XOR) operator. 
Thus, the random projection reduces the 
degrees of freedom of the model by $m$, as the first $m$ variables are completely determined by the last $n-m$. 
For later development it will also be convenient to rewrite these linear equations modulo 2 as polynomial 
equations by changing variables from $\{0,1\}$ to $\{-1,1\}$:
\begin{align}
(1-2 x_1) = \prod_{i=m+1}^n (1-2C_{1i} x_i) (1-2 b_1) ,  \cdots, (1-2 x_m) = \prod_{i=m+1}^n (1- 2C_{mi} x_i) (1-2 b_m)
\label{eq:constrained_vars}
\end{align}

\section{Combining Random Projections with I-Projections}
\label{main}
\begin{figure}
\centering
\includegraphics[scale=0.4]{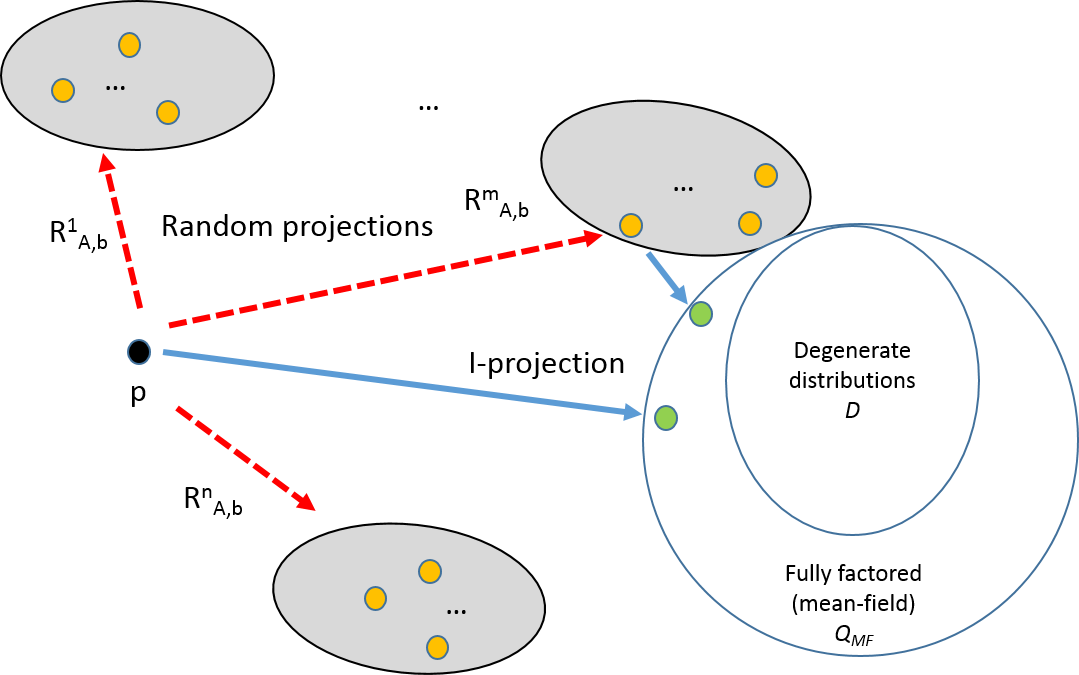}
\caption{Pictorial representation of the approach. $p$ is the target intractable distribution. Random 
projections implemented by universal hashing alter the partition function of the model in a predictable way. 
Random projections reduce the degrees of freedom of the model so that it becomes easier to 
approximate using tractable distributions.}
\label{fig:proj}
\end{figure}
Given an intractable target distribution $p$ and a candidate set of (tractable) distributions $\mathcal{Q}$, there are two main issues with variational approximation techniques: (i) $p$ can be far from the
approximating family $\mathcal{Q}$
in the sense that even the optimal $q^\ast = \arg\min_{q\in \mathcal{Q}}D_{KL}(q\|p)$ can have a large 
divergence $D_{KL}(q^\ast \|p)$ and therefore yield a poor lower bound in Eq. (\ref{varLB}), and (ii) the variational problem in Eq. (\ref{varLB}) is 
non-convex and thus difficult to solve exactly in high dimensions.
Our key idea is to address (i) by using the random projections introduced in the previous section 
to ``simplify'' $p$, producing a projection $R^m_{A,b}(p)$ that provably is closer to $\mathcal{Q}$. 
Crucially, because of the statistical properties of the random projection used,
(variational) inferences on the randomly projected model $R^m_{A,b}(p)$ reveal useful information about the original 
distribution $p$. Randomization plays a key role in our approach, boosting the power of variational inference. In fact, it is known that \#-P problems (e.g., computing the partition function) can be 
approximated in $\textrm{BPP}^{\textrm{NP}}$, i.e. in bounded error probabilistic polyonomial time 
by an algorithm that has access to a NP-oracle~\cite{stockmeyer1985approximation,valiant1986np,wishicml13}. 
Randomness appears to be crucial, and the ability to solve difficult (NP-equivalent) optimization 
problems, such as the one in in Eq. (\ref{varLB}), does not appear to be sufficient. By leveraging randomization, we are able to 
boost variational inference approaches (such as mean field), obtaining formal approximation 
guarantees for general target probability distributions $p$. A pictorial representation is given in Figure \ref{fig:proj}.

\subsection{Provably Tight Variational Bounds on the Partition Function}

Let $\mathcal{D} = \left\{ \delta x_0 | x_0\in X \right\}$ denote the set of degenerate 
probability distributions over $\{0,1\}^n$, i.e. probability distributions that 
place all the probability mass on a single configuration. There are $2^n$ such 
probability distributions and the entropy of each is zero. 
Given any probability distribution $p$, its projection on $\mathcal{D}$, i.e., $\arg \min_{q \in \mathcal{D}} D_{KL} (q||p)$, 
is given by a distribution that places all the probability on
$\arg \max_{x \in \mathcal{X}} \log p(x)$. Thus computing the I-projection 
on $\mathcal{D}$ is equivalent to solving a Most Probable Explanation 
query~\cite{koller2009probabilistic} which is NP-equivalent in the worst-case. 

Let $\mathcal{Q} \supseteq \mathcal{D}$ be a family of probability 
distributions that contains $\mathcal{D}$.
Our key result is that we can get \emph{provably tight bounds} on the partition function $Z$ by taking an I-projection onto $\mathcal{Q}$ \emph{after} a suitable random projection. This is formalized by the following theorem:

\begin{theorem}
\label{tightbounds}
Let $A^{i,t} \in \{0,1\}^{i \times n}\stackrel{iid}{\sim}\textnormal{Bernoulli}(\frac{1}{2})$ and $b^{i,t}\in \{0,1\}^i \stackrel{iid}{\sim}\textnormal{Bernoulli}(\frac{1}{2})$ for $i \in [0,n]$ and $t \in [1,T]$. 
Let $\mathcal{Q}$ be a family of distributions such that $\mathcal{D}\subseteq \mathcal{Q}$. Let 
\begin{equation}
\label{proj_var_inference_problem}
\gamma^{i,t} = \exp \left( \max_{q \in \mathcal{Q}} \theta\cdot \sum_{x: A^{i,t} x =b^{i,t}} q(x) \phi(x) - \sum_{x: A^{i,t}x =b^{i,t} } q(x) \log q(x) \right)
\end{equation}
be the optimal solutions for the projected variational inference problems.
Then for all $i \in [0,n]$ and for all $T\geq 1$ the rescaled variational solution 
is a lower bound for $Z$ in expectation
\[
\E\left[ \frac{1}{T} \sum_{t=1}^{T}  \gamma^{i,t} 2^i \right]=\E[\gamma^{i,t} 2^i ] \leq Z
\]
There also exists an $m$ such that for any $\delta > 0$ and positive constant $\alpha \leq 0.0042$, 
if $T \geq \frac{1}{\alpha} \left(\log (n /\delta)\right)$ then with 
probability at least $(1-2 \delta)$
\begin{eqnarray}
\label{mflb2}
\frac{1}{T} \sum_{t=1}^{T}  \gamma^{m,t} 2^m  \geq \frac{Z} {64 (n+1)}
\end{eqnarray}
\begin{eqnarray}
\label{mflb}
4 Z \geq Median\left(\gamma^{m,1}, \cdots,\gamma^{m,T} \right)  2^m 
\geq \frac{Z} {32 (n+1)}
\end{eqnarray}
\end{theorem}
\begin{proof}
See Appendix.
\end{proof}
This proves that appropriately rescaled variational lower bounds obtained on the 
randomly projected models (aggregated through median or mean) 
are within a factor $n$ of the true value $Z$, where $n$ is the number of 
variables in the model. This is an improvement on prior 
variational approximations which can either be unboundedly suboptimal or provide
guarantees that hold only in expectation \cite{zhu2015hybrid}; in contrast, our 
bounds are tight and require a relatively small number of samples proportional 
to $\log n/\delta$.
The proof, reported in the appendix for space reasons, relies on the following 
technical result which can be seen as a variational interpretation of Theorem 1 
from~\cite{wishicml13} and is of independent interest:
\begin{theorem}
\label{wishVariational}
For any $\delta > 0$, and positive constant $\alpha \leq 0.0042$, let $T \geq \frac{1}{\alpha} \left(\log (n /\delta)\right)$. Let $A^{i,t} \in \{0,1\}^{i \times n}\stackrel{iid}{\sim}\textnormal{Bernoulli}(\frac{1}{2})$ and $b^{i,t}\in \{0,1\}^i \stackrel{iid}{\sim}\textnormal{Bernoulli}(\frac{1}{2})$ for $i \in [0,n]$ and $t \in [1,T]$. Let
\[
\delta^{i,t} = \min_{q \in \mathcal{D}} D_{KL} (q||R^i_{A^{i,t},b^{i,t}}(p))
\]
Then with probability at least $(1-\delta)$
\begin{align}
\label{Zbounds}
\sum_{i=0}^n \exp \left( Median\left(-\delta^{i,1}+ \log Z(A^{i,1},b^{i,1}), \cdots, - \delta^{i,T} + \log Z(A^{i,T},b^{i,T}) \right) \right) 2^{i-1}
\end{align}
is a 32-approximation to $Z$.
\end{theorem}
The proof can be found in the appendix. Intuitively, Theorem \ref{wishVariational} states that one can always find a small number of
degenerate distributions (which can be equivalently thought of as special states that can be discovered through
random projections and KL-divergence minimization) that are with high probability representative of
the original probabilistic model, regardless of how complex the model is. Theorem \ref{tightbounds}
extends this idea to more general families of distributions such as Mean Field.

\subsection{Solving Randomly Projected Variational Inference Problems}
To apply the results from Theorem \ref{tightbounds} we must choose a tractable approximating family $\mathcal{D}\subseteq \mathcal{Q}$ for the I-projection part and integrate our random projections into the optimization procedure.
We focus on mean field ($\mathcal{Q}=\mathcal{Q}_{MF}$) as our approximating family, 
but the results can be easily extended to structured mean field 
\cite{bouchard2009optimization}. For simplicity of exposition we consider only probabilistic 
models $p$  with unary and binary factors (e.g. Ising models, restricted Boltzmann machines). 
That is, $p(x) = \exp(\theta\cdot \phi(x)) / Z$, where $\phi(x)$ are single node and pairwise 
edge indicator variables.

Recall that our projection $R^m_{A,b}(p)$ constrains the distribution $p$ to $\left\{ x|Ax=b\bmod 2
\right\}$.  The projected variational optimization problem (\ref{proj_var_inference_problem}) is therefore
\[
\log Z(A,b) \geq \max_{q}\ \  \theta \cdot \sum_{x|Ax=b\bmod 2} q(x) \phi(x) - \sum_{x|Ax=b\bmod 2} q(x) \log q(x) 
\]
Or, equivalently,
\begin{equation}
\label{objfunc_eq}
\log Z(A,b) \geq \max_{\mu}\ \  \theta \cdot \mu + \sum_{i=m+1}^n  H(\mu_i)
\end{equation}
where $\mu$ is the vector of singleton and pairwise marginals of $q(x)$ 
and $H(\mu_i)$  is the entropy of a Bernoulli random variable with parameter $\mu_i$. 
To solve this optimization problem efficiently
we need a clever way to take into account the parity constraints, for
running traditional mean field with message passing
as in \cite{zhu2015hybrid} would fail in the normalization step
because of the presence of hard parity constraints.
The key idea is to consider 
the equivalent row-reduced representation of the constraints from (\ref{eq:constrained_vars}) and define
\[
q(x_1, \cdots, x_n) = \prod_{i=m+1}^n q_i(x_i) \prod_{k=1}^m 1\left\{ (1-2 x_k) = \prod_{i=m+1}^n (1-2C_{ki} x_i) (1-2 b_k) \right\}
\]
where we have a set of independent ``free variables'' (wlog., the last $n-m$) and a set of ``constrained variables'' (the first $m$) that are always set as to satisfy the parity constraints. 
Since the variables 
$x_1,\ldots,x_m$ are fully
determined by $x_{m+1},\ldots,x_n$, we see that the marginals $\mu_1,\ldots,\mu_m$ are also
determined by $\mu_{m+1},\ldots,\mu_n$, as shown by the following proposition:
\begin{proposition}
\label{marginals_worked_out_prop}
The singleton and pairwise marginals in (\ref{objfunc_eq}) can be computed as follows:

Singleton marginals: for $k \in [m+1, n]$,  $\mu_k = E_q \left[ x_k \right] = q_k(1)$.
For $k \in [1, m]$,
\[
\mu_k  = \left(1-(1-2 b_k)\prod_{i=m+1}^n (1- 2C_{ki} \mu_i)  \right)/2
\]
Pairwise marginals: for $k,\ell \in [m+1,n]$, $\mu_{kl} = E_q [ x_k x_\ell] = \mu_k \mu_\ell$. For $k \in [m+1,n]$, $\ell \in [1,m]$
\[
 \mu_{kl} =  
\begin{cases}
    \mu_k \frac{1}{2} (1 + (1-2b_l)\prod_{i \neq k, i=m+1}^n (1 - 2C_{li}\mu_i)) & \text{if }
        C_{lk}=1 \\
        \mu_k \mu_l & \text{otherwise}
\end{cases}
\] 
For $k \in [1,m]$, $\ell \in [1,m]$
\[\begin{aligned}
\mu_{kl} = \frac{1}{4} (&1 + (1 - 2b_k)(1 - 2b_l)\prod_{i=m+1}^n (1 - \mu_i(2C_{ki} + 2C_{li} - 4C_{ki}C_{li}))\\ 
& -  (1-2b_k)\prod_{i=m+1}^n (1 - 2C_{ki}\mu_i)  -   (1-2b_l)\prod_{i=m+1}^n (1 - 2C_{li}\mu_i))
\end{aligned}\]
\end{proposition}
The derivation is found in the appendix. We can therefore maximize the lower bound in (\ref{objfunc_eq}) by optimizing only over the ``free marginals'' $\mu_{m+1},\ldots,\mu_n$, as the remaining one are completely determined per Proposition \ref{marginals_worked_out_prop}. Compared to a traditional mean field variational approximation, we have a problem with a smaller number of variables, but with additional non-convex constraints.

\section{Algorithm: Mean Field with Random Projections}
\label{algorithm}
Theorem \ref{tightbounds} guarantees that the approximation to $Z$ has a tight lower bound only if we 
are able to find globally optimal solutions for  (\ref{objfunc_eq}).
However, nontrivial variational inference problems (\ref{varLB}) are non-convex in general even without any
random projections and even when $\mathcal{Q}$ is simple, e.g., $\mathcal{Q}=\mathcal{Q}_{MF}$.
We do not explicitly handle this nonconvexity, but nevertheless we show empirically that 
we can vastly improve on mean field lower bounds.  The key insight for our optimization procedure is that the objective function is still coordinate-wise concave, like in a traditional mean-field approximation:
\begin{proposition}
\label{marginals_concave_prop}
The objective function $\theta \cdot \mu + \sum_{i=m+1}^n  H(\mu_i)$ in (\ref{objfunc_eq}) is concave with respect 
to any particular free marginal $\mu_{m+1},\ldots,\mu_n$.
\end{proposition}
\begin{proof}
  By inspection, all the marginals in Proposition \ref{marginals_worked_out_prop} 
  are linear with respect to any specific 
  free marginal $\mu_{m+1},\ldots,\mu_n$. 
  Since the entropy term is concave, the RHS in (\ref{objfunc_eq})
  is concave in each free marginal $\mu_{m+1},\ldots,\mu_n$.
\end{proof}

Since (\ref{objfunc_eq}) is concave in each variable we devise a coordinate-wise ascent 
algorithm, called Mean Field with Random Projections (MFRP),
for maximizing the lower bound in (\ref{objfunc_eq}) over the free marginals defined
in Proposition \ref{marginals_worked_out_prop}. Starting from a random initialization, we iterate over each free marginal $\mu_k$ and maximize (\ref{objfunc_eq})
with the rest of the free marginals fixed by setting the gradient with respect to $\mu_k$ equal
to 0 and solving for $\mu_k$. The closed form expressions we use are reported in the appendix.
Because the overall optimization problem is not concave the algorithm may converge
at a local maximum; therefore, we use $J$ random initializations
and use the best lower bound found across the $J$ runs of the ascent algorithm. 
For a given $m$, we repeat this procedure $T$ times and return the median across the runs.
Each coordinate ascent step for free marginal $\mu_i$ takes $O(m+n+|E_{cc}|(n-m))$ steps 
in expectation where $E_{cc}$ is 
the number of variables co-occurring in a parity constraint.
Recomputing the constrained marginals takes $O(m(n-m))$ steps. 

The algorithm returns the maximum of MFRP$(p(x), m)$ over $m\in[0,n]$. If MFRP finds a global
optimum, then Theorem \ref{tightbounds} guarantees it is a tight lower bound for $\log Z$ with high
probability. Since MFRP uses coordinate-wise ascent we cannot certify global optimality; however,
our experiments show large improvements in the lower bound when compared to existing variational
methods.

\begin{algorithm}
\caption{MFRP$(p(x) \propto \prod_\alpha \psi_\alpha(\{x\}_\alpha),m)$ }
\label{algoGibbs}
\begin{algorithmic}
\For {$t=1, \cdots, T$}	\Comment{Do $T$ random projections}
    \State Generate parity bits $b^{(t)} \stackrel{iid}{\sim}\textnormal{Bernoulli}(\frac{1}{2})^m$ \Comment{Generate random projection $R^m_{A^{(t)},b^{(t)}}$}
    \State Generate matrix $A^{(t)}\stackrel{iid}{\sim}\textnormal{Bernoulli}(\frac{1}{2})^{m\times n}$
    \State Row reduce $A$,$b$ and permute to yield $C=[I | A']$ and $b$ \Comment{Compute constraints}
    \State $\widetilde{Z}^{(t)} \leftarrow 0$
    \For {$j=1, \cdots, J$}		
     \Comment{Try different initializations}				
        \State Initialize $\mu^{(j, t)}\stackrel{iid}{\sim}\textnormal{Unif}(0, 1)^n$
        \For {$l=1, \cdots, m$} 		
        \Comment{Compute constrained marginals}											
        \State $\mu^{(j, t)} \leftarrow \left(1-\prod_{i=m+1}^n (1- 2C_{li} \mu^{(j, t)}_i)
                (1-2 b_l) \right)/2$
        \EndFor
        \While {not converged} \Comment{Stop when the increment is small or timeout}
        \For {$k=m+1, \cdots, n$}
        \Comment{Coordinate ascent over free marginals}
        \State $\mu^{(j,t)}_k \leftarrow \arg\max_{\mu_k} \theta\cdot\mu^{(j, t)} + \sum_{i=m+1}^n H(\mu^{(j, t)}_i)$
            \For {$l=1, \cdots, m$} \Comment{Update constrained marginals}
                \State $\mu^{(j, t)}_l \leftarrow 
                \left(1-\prod_{i=m+1}^n (1- 2C_{li} \mu^{(j, t)}_i) (1-2 b_l) \right)/2$
            \EndFor
        \EndFor
        \EndWhile
    \EndFor
      \State $\widetilde{Z}^{(t)} \leftarrow \max_j
          \exp(\theta \cdot \mu^{(j, t)} + \sum_{i=m+1}^n  H(\mu^{(j, t)}_i))
          $
					\Comment{Pick best over initializations}
\EndFor
\State Return $ 2^m Median\left (  \widetilde{Z}^{(1)}, \ldots, \widetilde{Z}^{(T)}  \right) $ \Comment{Aggregate across projections}
\end{algorithmic}
\end{algorithm}

\section{Experiments }
\label{experiments}
We investigate MFRP's empirical performance on Ising models and on Restricted Boltzmann Machines.
In particular, we are interested in the log partition function estimates and in the quality of the
marginal estimates. Where applicable, exact ground truth estimates are obtained with the libDAI 
implementation of Junction Tree \cite{mooij2010libdai}. Upper bounds are calculated 
with Tree-Reweighted Belief Propagation (TRW-BP) \cite{wainwright2003tree},
also implemented in libDAI. All methods are compared to mean field (MF) optimized 
with coordinate-wise ascent and random restarts.

\subsection{Ising Models}

We consider $n\times n$ binary grid Ising models with variables $x_i\in \{-1,1\}$ and potentials
$\psi_{ij}(x_i,x_j)=\exp(w_{ij}x_i x_j + f_i x_i + f_j x_j)$. 
In particular, we look at mixed models where the $w_{ij}$'s are drawn uniformly from $[-10,10]$ and 
the $f_i$'s uniformly from $[-1,1]$. 

Figure \ref{fig:ising} compares the log partition function estimates from MF, Junction Tree, MFRP,
and TRW-BP. For each grid size, we generated
five different grids and computed the mean field estimate for each as a baseline lower bound. 
For each of the five grids we also computed the best MFRP lower bound over $m\in[0,20]$ 
with $T=5$ trials each. 
For comparison we include the exact log partition calculation from Junction Tree up to $n=20$ and the TRW-BP upper bounds for all $n$. 
We plot the mean and standard error bars of the log ratio of each estimate over mean field for each method over the five grids. 
Note that for large grid sizes, the lower bound provided by MFRP is hundreds of orders of magnitude
better than than those found by mean field.

Finally, we consider the empirical runtime of the method for varying grid sizes $n$ and number of
constraints $m$ in Figure \ref{fig:ising_runtime}. As expected, the runtime for mean field grows linearly
in the number of variables in the graph (quadratically with $n$) and there is a linear slowdown as
constraints are added to the optimization.

\begin{figure}[h]
    \centering
    \begin{minipage}[t]{.45\textwidth}
        \centering
        \includegraphics[width=\textwidth]{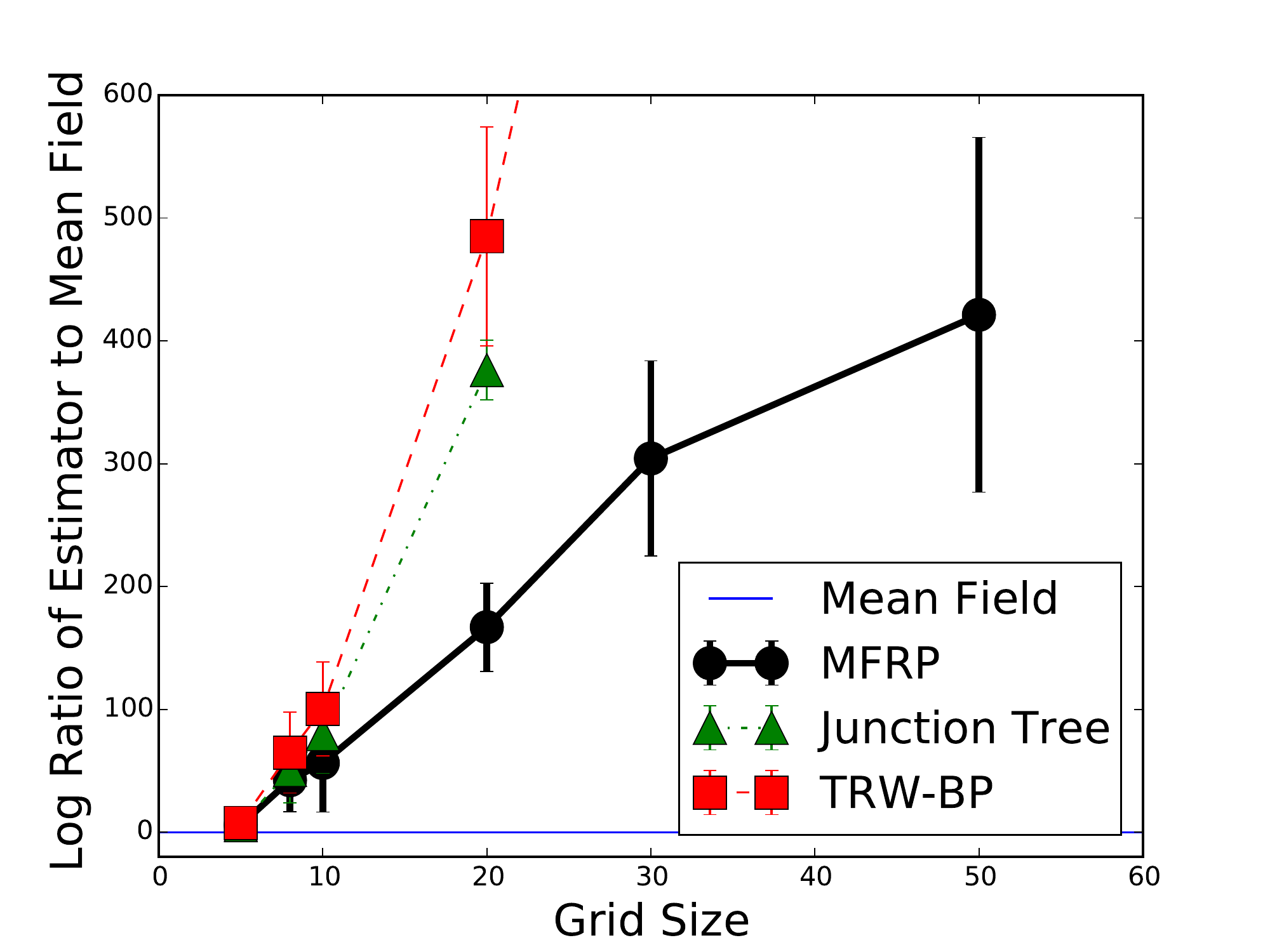}
        \caption{Ising grids: for each size, we plot the
            ratios of the estimates from each method to the mean field estimate, with standard error bars 
            based on 5 runs. Results are reported as a log ratio compared to mean field.}
        \label{fig:ising}
    \end{minipage}
    \hspace{.5cm}
    \begin{minipage}[t]{.45\textwidth}
        \centering
        \includegraphics[width=\textwidth]{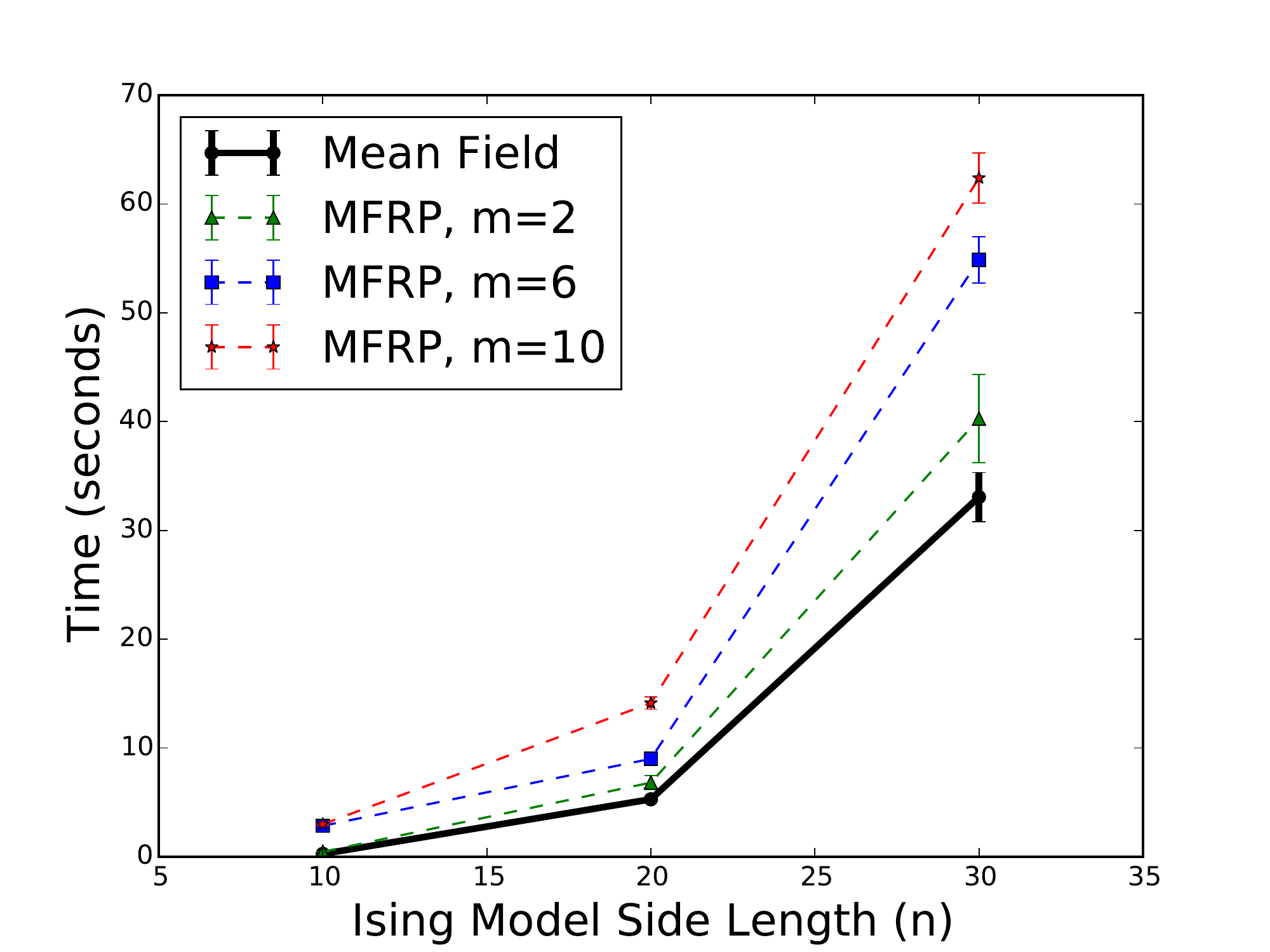}
        \caption{Runtimes for MFRP across different grid sizes and constraints. Both methods are
        linear in the number of variables in the model ($n^2$ in the Ising model case). }
        \label{fig:ising_runtime}
    \end{minipage}
\end{figure}

\subsection{Restricted Boltzmann Machines}

We train Restricted Boltzmann Machines (RBMs) \cite{hinton2006fast} using Contrastive Divergence
(CD)~\cite{welling2002new,carreira2005contrastive} on the MNIST hand-written digits dataset. 
In an RBM there is a layer of $n_h$ hidden binary variables $h=h_1, \cdots, h_{n_h}$ and a 
layer of $n_v$ binary visible units $v=v_1, \cdots, v_{n_v}$. The joint probability distribution is given by
$
P(h,v) = \frac{1}{Z} \exp(b'v+c'h+h'Wv)
$.
We use $n_h\in \{100, 200\}$ hidden units and $n_v=28\times 28=784$ visible units. 
We learn the parameters $b,c,W$ using CD-$k$ for $k \in \{1,5,15\}$, where $k$ 
denotes the number of Gibbs sampling steps used in the inference phase, 
with $15$ training epochs and minibatches of size $20$.

We then use MF and MFRP to estimate the log partition function and also consider the aggregate
marginals of the sub-problems. For most of the cases we see a clear improvement in both the log partition lower bounds and in the marginals,
with the marginal for No. Hidden Nodes $=100, k=15$ similar visually to an average over all digits in the MNIST dataset.

\begin{table}[htdp]
\begin{center}
\begin{tabular}{|c|c|c|c|c|c|c|}
\hline
No. Hidden Nodes & 100 & 100 & 100 & 200 & 200 & 200\\
\hline
CD-$k$ & 1 & 5 & 15 & 5 & 15 & 25\\
\hline
MF $\log Z$ & 501 & 348 & 297 & 203& 293 & 279\\
\hline
MFRP $\log Z$ & 501 & {\bf 433} & {\bf 342} & {\bf 380} & {\bf 313} & {\bf 295} \\
\hline
MF $\mu$ &  \includegraphics{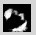} & \includegraphics{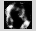} & \includegraphics{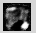} & \includegraphics{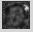} & \includegraphics{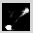} & \includegraphics{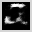} \\
\hline
 MFRP $\mu$ & \includegraphics{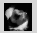}& \includegraphics{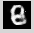}& \includegraphics{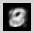}& \includegraphics{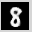}& \includegraphics{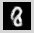}& \includegraphics{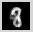}\\
 \hline
\end{tabular}
\end{center}
\caption{$\log Z$ and $\mu$ estimates across two RBMs trained with different sampling parameters.}
\label{table:rbm}
\end{table}

\section{Conclusions}
\label{conclusions}
We introduced a new, general approach to variational inference that combines random projections with I-projections to obtain provably tight lower
bounds for the log partition function. Our approach is the first to leverage universal hash functions and their properties in a variational sense.
We demonstrated the effectiveness of this idea by extending mean field with random projections and empirically showed a large 
improvement in the partition function lower bounds and marginals obtained on both synthetic and real world data.
Natural extensions to the approach include applications to other variational methods, like the Bethe approximation, and the use of more powerful global optimization techniques instead of the coordinate-wise ascent currently used.

\newpage

\bibliographystyle{unsrt}
\small
\bibliography{VariationalProjections}
\clearpage
\newpage
\appendix
\section{Appendix : Proofs}
\begin{proof}[Proof of Theorem~\ref{wishVariational}]
\begin{align*}
\min_{q \in \mathcal{D}} D_{KL} (q||R^i_{A^{i,t},b^{i,t}}(p)) =\\
\min_{q \in \mathcal{D}} \sum_{x | A^{i,t}x =b^{i,t} \bmod 2} q(x) \log q(x) - \theta \cdot \sum_{x | A^{i,t}x =b^{i,t} \bmod 2} q(x) \phi(x) + \log Z(A^{i,t},b^{i,t}) = \\
\min_{q \in \mathcal{D}} - \theta \cdot \sum_{x | A^{i,t}x =b^{i,t} \bmod 2} q(x) \phi(x) + \log Z(A^{i,t},b^{i,t}) = \\
- \max_{x | A^{i,t}x =b^{i,t} \bmod 2} \theta \phi(x) + \log Z(A^{i,t},b^{i,t})
\end{align*}
Therefore
\[
- \min_{q \in \mathcal{D}} D_{KL} (q||R^i_{A^{i,t},b^{i,t}}(p)) + \log Z(A^{i,t},b^{i,t}) = \max_{x | A^{i,t}x =b^{i,t} \bmod 2} \theta \phi(x)
\]
Substituting into  Eq. (\ref{Zbounds}) we can rewrite as
\begin{align*}
\nonumber
\sum_{i=0}^n \exp \left( Median\left( \max_{x | A^{i,1}x =b^{i,1} \bmod 2} \theta \phi(x), \cdots, \max_{x | A^{i,T}x =b^{i,T} \bmod 2} \theta \phi(x) \right) \right) 2^{i-1} = \\
\sum_{i=0}^n   Median\left( \exp \left( \max_{x | A^{i,1}x =b^{i,1} \bmod 2} \theta \phi(x) \right), \cdots, \exp \left( \max_{x | A^{i,T}x =b^{i,T} \bmod 2} \theta \phi(x) \right) \right) 2^{i-1} = \\
\sum_{i=0}^n   Median\left( \max_{x | A^{i,1}x =b^{i,1} \bmod 2} \exp (\theta \phi(x)), \cdots, \max_{x | A^{i,T}x =b^{i,T} \bmod 2} \exp (\theta \phi(x)) \right)  2^{i-1}
\end{align*}
The result then follows directly from Theorem 1 from~\cite{wishicml13}.
\end{proof}

\begin{proof}[Proof of Theorem~\ref{tightbounds}]
For the first part, we have the following relationship on the expectation:
\[
\E[Z(A^{i,t},b^{i,t})] = 2^{-i} Z
\]
From the non-negativity of the KL divergence we have that for any $q \in \mathcal{Q}$
\[
\log Z(A^{i,t},b^{i,t}) \geq \theta \cdot \sum_{x: A^{i,t} x =b^{i,t}} q(x) \phi(x) - \sum_{x:A^{i,t} x =b^{i,t}} q(x) \log q(x) 
\]
\[
\log Z(A^{i,t},b^{i,t}) \geq \max_{q \in \mathcal{Q}} \left \{  \theta \cdot \sum_{x: A^{i,t} x =b^{i,t}} q(x) \phi(x) - \sum_{x: A^{i,t} x =b^{i,t}} q(x) \log q(x) \right \}
\]
\begin{equation}
\label{Zgammalb}
Z(A^{i,t},b^{i,t}) \geq \exp \left( \max_{q \in \mathcal{Q}} \left \{  \theta \cdot \sum_{x: A^{i,t}
x =b^{i,t}} q(x) \phi(x) - \sum_{x: A^{i,t} x =b^{i,t}} q(x) \log q(x) \right \} \right) \triangleq \gamma^{i,t}
\end{equation}
We take the expectation of both sides to yield
\[
\frac{Z}{2^i}= \E \left[ Z(A^{i,t},b) \right] \geq \E \left[ \exp \left( \max_{q \in \mathcal{Q}} \left \{  \theta \cdot \sum_{x:A^{i,t} x =b} q(x) \phi(x) - \sum_{x:A^{i,t} x =b} q(x) \log q(x) \right \} \right)\right] = \E[\gamma^{i,t} ]
\]
For the second part. Since the conditions of Theorem \ref{wishVariational} are satisfied, we know
that equation (\ref{Zbounds}) holds with probability at least $1-\delta$. From (\ref{Zbounds}),
since the terms in the sum are non-negative we have that the maximum element is at least $1/(n+1)$
of the sum:
\begin{align*}
\max_i \exp \left( Median\left( - \min_{q \in \mathcal{D}} D_{KL} (q||R^i_{A^{i,1},b^{i,1}}(p)) + \log Z(A^{i,1},b^{i,1}), \cdots,  \right. \right. \\
 \left. \left. - \min_{q \in \mathcal{D}} D_{KL} (q||R^i_{A^{i,T},b^{i,T}}(p)) + \log Z(A^{i,T},b^{i,T}) \right) \right) 2^{i-1} \geq \frac{1}{32} Z \frac{1}{n+1}
\end{align*}
Therefore there exists $m$ such that
\begin{align*}
Median\left( - \min_{q \in \mathcal{D}} D_{KL} (q||R^m_{A^{m,1},b^{m,1}}(p)) + \log Z(A^{m,1},b^{m,1}), \cdots,   \right. \\
 \left. - \min_{q \in \mathcal{D}} D_{KL} (q||R^m_{A^{m,T},b^{m,T}}(p)) + \log Z(A^{m,T},b^{m,T}) \right) + (m-1) \log 2 \geq -\log 32  + \log Z- \log (n+1)
\end{align*}
We also have
\[
\min_{q \in \mathcal{Q}} D_{KL} (q||R_{A,b}(p)) \leq \min_{q \in \mathcal{D}} D_{KL} (q||R_{A,b}(p))
\]
because $\mathcal{D} \subseteq \mathcal{Q}$. 
Thus
\begin{align*}
Median\left( - \min_{q \in \mathcal{Q}} D_{KL} (q||R^m_{A^{m,1},b^{m,1}}(p)) + \log Z(A^{m,1},b^{m,1}), \cdots,   \right. \\
 \left. - \min_{q \in \mathcal{Q}} D_{KL} (q||R^m_{A^{m,T},b^{m,T}}(p)) + \log Z(A^{m,T},b^{m,T}) \right) + (m-1) \log 2 \geq -\log 32  + \log Z- \log (n+1)
\end{align*}
From the definition of $D_{KL}$, we have
\[
D_{KL} (q||R^m_{A,b}(p)) = -\sum_{x: A x =b \bmod 2} q(x) \phi(x) + \sum_{x: A x =b \bmod 2} q(x) \log q(x) + \log Z(A,b)
\]
Pluggin in we get
\begin{eqnarray*}
Median\left( \log \gamma^{m,1} , \cdots,  \log \gamma^{m,T} \right) + (m-1) \log 2 \geq -\log 32 - \log (n+1) + \log Z
\end{eqnarray*}
and also
\begin{eqnarray*}
Median\left(\log  \gamma^{m,1} , \cdots, \log  \gamma^{m,T} \right) + m\log 2 \geq -\log 32 - \log (n+1) + \log Z
\end{eqnarray*}
with probability at least $1-\delta$.
\begin{eqnarray*}
Median\left(  \gamma^{m,1} , \cdots,  \gamma^{m,T} \right) 2^m \geq \frac{Z} {32 (n+1)}
\end{eqnarray*}
Since the terms are non zero,
\begin{eqnarray*}
\frac{1}{T} \sum_{t=1}^{T} \gamma^{m,t} \geq \frac{1}{2}  Median\left(  \gamma^{m,1} , \cdots,  \gamma^{m,T} \right)
\end{eqnarray*}
therefore with probability at least $1-\delta$
\begin{eqnarray*}
\frac{1}{T} \sum_{t=1}^{T} \gamma^{m,t} 2^m \geq \frac{Z} {64 (n+1)}
\end{eqnarray*}

From Markov's inequality
\[
\P \left[  Z(A^{i,t},b^{i,t}) \geq c \E[Z(A^{i,t},b^{i,t})]    \right] \leq \frac{1}{c}
\]
\[
\P \left[  Z(A^{i,t},b^{i,t}) 2^i \geq c Z    \right] \leq \frac{1}{c}
\]
Therefore since $Z(A^{i,t},b^{i,t}) \geq  \gamma^{i,t}$ from (\ref{Zgammalb}), setting $c=4$ and $i=m$ we get
\[
\P \left[  \gamma^{m,t} 2^m \geq 4 Z    \right] \leq \frac{1}{4}
\]
From Chernoff's inequality,
\[
\P \left[  4 Z \geq Median\left(\gamma^{m,1}, \cdots,\gamma^{m,T} \right)  2^m    \right] \geq 1-\delta
\]
and the claim follows from union bound.
\end{proof}

\begin{proof}[Proof of Proposition~\ref{marginals_worked_out_prop}]
For singleton marginals, when $k \in [m+1, n]$, $x_k$ is a free variable and thus $\mu_k = E_q[x_k] = q_k(1)$. When $k \in [1, m]$, 
\[
(1-2 x_k) = (1-2 b_k) \prod_{i=m+1}^n (1-2C_{ki} x_i) 
\]
Take the expectation on both side and since $x_i$ for $i \in [m+1, n]$ are free (independent) variables, we have 
\[
(1-2 \mu_k) =  (1-2 b_k)  \prod_{i=m+1}^n (1-2C_{ki} \mu_i)
\]
That is,
\[
\mu_k  = \left(1-(1-2 b_k)\prod_{i=m+1}^n (1- 2C_{ki} \mu_i)  \right)/2
\]
For the binary marginal $\mu_{kl}$, there are three cases: both $x_k, x_l$ are free variables; one is free and the other is constrained; both are constrained. For the first case, $k, \ell \in [m+1, n]$, they are independent and thus 
\[
\mu_{kl} = E_q [ x_k x_\ell] = \mu_k \mu_\ell
\]
For the second case, $k \in [m+1,n]$, $\ell \in [1,m]$. 
\[
\mu_{kl} = Pr[x_k=1, x_l=1] = \sum_{
\begin{array}{c}
x_{m+1}, ... , x_n, \\x_k = 1\\
 -1 = (1-2x_l) = (1-2b_l)\prod_{i=m+1}^n (1 - 2C_{li}x_i)
 \end{array}} \prod_{i=m+1}^n q_i(x_i)
\]
When $C_{lk} = 1, 1 - 2C_{lk}x_k = -1$, 
\[
\mu_{kl} =  \mu_k \cdot \sum_{
\begin{array}{c}
x_{m+1}, ...x_{k-1}, x_{k+1},... , x_n, \\
 1 = (1-2b_l)\prod_{i=m+1, i\neq k}^n (1 - 2C_{li}x_i)
 \end{array}} \prod_{i=m+1, i\neq k}^n q_i(x_i)
\]
Let's introduce a new binary variable, u, satisfying the constraint 
\[
(2u-1) = (1-2b_l)\prod_{i \neq k, i=m+1}^n (1 - 2C_{li}x_i)
\]
Then the above summation is over $x_{m+1}, ... x_{k-1}, x_{k+1}, ..., x_n, u$ such that $u = 1$. 
The probability of u being 1 is the expected value of u. Therefore, 
\[
\mu_{kl} = \mu_kE[u] = \mu_k \frac{1}{2} (1 + (1-2b_l)\prod_{i \neq k, i=m+1}^n (1 - 2C_{li}\mu_i))
\]
When $C_{lk} = 0, x_l$ is independent of $x_k$, so $\mu_{kl} = \mu_k\mu_l$\\
For the last case, $k, \ell \in [1,m]$. 
\[
(1-2 x_k) (1-2x_\ell) = (1-2 b_k) (1-2 b_\ell) \prod_{i=m+1}^n (1-2C_{ki} x_i)  \prod_{i=m+1}^n (1-2C_{\ell i} x_i)
\]
Taking the expected value of both side
\[
1 - 2\mu_l - 2\mu_k + 4\mu_{kl} = (1-2b_l)(1-2b_k)\prod_{i=m+1}^n E[1 - x_i(2C_{ki} + 2C_{li}) + 4C_{ki}C_{li}x_i^2]
\]
\[
\mu_{kl} = \frac{1}{4} (-1 + 2\mu_k + 2\mu_l + (1 - 2b_k)(1 - 2b_l)\prod_{i=m+1}^n (1 - \mu_i(2C_{ki} + 2C_{li} - 4C_{ki}C_{li})))
\]
Plug in the result of $\mu_k, \mu_l$:
\[
\begin{aligned}
\mu_{kl} = \frac{1}{4} (&1 + (1 - 2b_k)(1 - 2b_l)\prod_{i=m+1}^n (1 - \mu_i(2C_{ki} + 2C_{li} - 4C_{ki}C_{li}))\\ 
& -  (1-2b_k)\prod_{i=m+1}^n (1 - 2C_{ki}\mu_i)  -   (1-2b_l)\prod_{i=m+1}^n (1 - 2C_{li}\mu_i))
\end{aligned}
\]

\end{proof}

\begin{proposition}[The gradients for coordinate ascent]
\label{gradient_prop}

Assuming we are taking the gradient with respect to $\mu_k$, where $k \in [m+1, n]$.  \\
1. Unary term $\mu_k$  
\[
\frac{\partial \mu_k}{\partial \mu_k} = 1
\]
And thus 
\[
\frac{\partial H(\mu_k)}{\partial \mu_k} = \frac{\partial}{\partial \mu_k} -(\mu_k \log(\mu_k) + (1 - \mu_k) \log(1 - \mu_k)) = \log(\frac{1 - \mu_k}{\mu_k})
\]
2. Unary term $\mu_l$, $l \geq m+1, l \neq k$
\[
\frac{\partial \mu_l}{\partial \mu_k} = 0
\]
\[
\frac{\partial H(\mu_l)}{\partial \mu_k} = 0
\]
3. Unary term $\mu_l$, $l \leq m$
\[
\frac{\partial \mu_l}{\partial \mu_k} = \frac{\partial}{\partial \mu_k} \frac{1}{2} (1 - (1 - 2b_l) \prod_{i=m+1}^n (1 - 2C_{li}\mu_i)) = 
(1 - 2b_l)C_{lk} \prod_{i=m+1, i \neq k}^n (1 - 2C_{li}\mu_i)
\]

4. Binary term, $\mu_{kl}, l \geq m+1$
\[
\frac{\partial}{\partial \mu_k} \mu_{kl} = \mu_l
\]
5. Binary term, $\mu_{pl}$, both $p, l \geq m+1, p \neq k, l \neq k$
\[
\frac{\partial}{\partial \mu_k} \mu_{pl} = 0
\]
6. Binary term, $\mu_{kl}, l \leq m$\\
When $C_{lk} = 0$, $\mu_{kl} = \mu_{k}\mu_l$ and its derivative is
\[
\mu_l + \mu_k\mu_l' = \mu_l = \frac{1}{2} (1 -  (1-2b_l)\prod_{i=m+1, i\neq k}^n (1 - 2C_{li}\mu_i))
\]
When $C_{lk} = 1$, $\mu_{kl} =  \mu_k \frac{1}{2} (1 + (1-2b_l)\prod_{i \neq k, i=m+1}^n (1 - 2C_{li}\mu_i))$. The derivative is 
\[
 \frac{1}{2} (1 + (1-2b_l)\prod_{i \neq k, i=m+1}^n (1 - 2C_{li}\mu_i))  
 \]
7. Binary term, $\mu_{pl}$, where $p \geq m+1, p \neq k, l \leq m$\\
When $C_{lp} = 0$, $\mu_{pl} = \mu_{p}\mu_l$ and its derivative is
\[
\mu_p \mu_l'
\]
When $C_{lp} = 1$, $\mu_{pl} =  \mu_p \frac{1}{2} (1 + (1-2b_l)\prod_{i \neq p, i=m+1}^n (1 - 2C_{li}\mu_i))$. The derivative is 
\[
 -\mu_p C_{kl} (1-2b_l)\prod_{i \neq k, i\neq p, i=m+1}^n (1 - 2C_{li}\mu_i))  
 \]
8 Binary term $\mu_{pl}$, where both $p, l \leq m$
\[\begin{aligned}
\frac{\partial}{\partial \mu_k} \mu_{pl} = \frac{\partial}{\partial \mu_k} \frac{1}{4} (1 &-  (1 - 2b_p)\prod_{i=m+1}^n(1 - 2C_{pi}\mu_i)\\
&-  (1 - 2b_l)\prod_{i=m+1}^n(1 - 2C_{li}\mu_i) \\
&+  (1 - 2b_p)(1 - 2b_l)\prod_{i=m+1}^n(1 - (2C_{pi} + 2C_{li} - 4c_{pi}C_{li})\mu_i))
\end{aligned}\]
\[\begin{aligned}
\frac{\partial}{\partial \mu_k} \mu_{pl} = &\frac{(1 - 2b_p)C_{pk}}{2} \prod_{i=m+1, i\neq k}^n (1 - 2C_{pi}\mu_i) \ \ \ ( = \mu_p'/2)\\
&+  \frac{(1 - 2b_l)C_{lk}}{2} \prod_{i=m+1, i\neq k}^n (1 - 2C_{pi}\mu_i) \ \ \ ( = \mu_l'/2)\\
&- \frac{(1-2b_p)(1 - 2b_l)(2C_{pk} + 2C_{lk} - 4C_{pk}C_{lk})}{4} \prod_{i=m+1, i \neq k}^n (1 - \mu_i(2C_{pi} + 2C_{li} - 4C_{pi}C_{li}))
\end{aligned}\]

All gradients except the entropy one are independent of $\mu_k$, so whole gradient can be expressed as 
\[
c +  \log(\frac{1 - \mu_k}{\mu_k}),
\]
where c is a constant with respect to $\mu_k$. Therefore, the coordinate ascent step for $\mu_k$ is to set it to 
\[
\frac{1}{1 + \exp(-c)}
\]
\end{proposition}

\end{document}